\documentclass[11pt]{article}    
\usepackage{tgtermes}
\usepackage{appendix}
\usepackage{enumerate}
\usepackage{bbm}
\usepackage{times}
\usepackage{tcolorbox}
\usepackage{adjustbox}
\usepackage{mathtools}
\newtheorem{lemma}{Lemma}
\newtheorem{theorem}{Theorem}
\newtheorem{proposition}{Proposition}
\newtheorem{definition}{Definition}

\newcounter{remno} 
{}
\newenvironment{proof}{\noindent {\bf Proof. }}{\hfill  \\}
\newfont{\mfoo}{cmssdc10 scaled\magstep1}
\newfont{\mfo}{cmtt9 scaled\magstep1}

\usepackage{fancyhdr}
\usepackage{pdfpages}\graphicspath{{/Users/mnk/Pictures/}}
\fancyheadoffset[LE,RO]{\marginparsep+\marginparwidth}
\usepackage{amsmath}
\usepackage{amssymb}
\usepackage{mathptmx}      
\usepackage{color}
\usepackage{mathdots}
\usepackage{dsfont}
\usepackage{pdfsync}
\usepackage{enumitem}
\usepackage{mathtools}
\usepackage{subfig}
\usepackage{endnotes}
\usepackage[numbers]{natbib}
\DeclarePairedDelimiter{\ceil}{\lceil}{\rceil}

\newcommand{\cc}{\citet}
 
\newcommand{\eqqref}[1]{Eq.\,\eqref{#1}}

\newcommand{\inqqref}[1]{Ineq.\,\eqref{#1}}

\newcommand{\ismi}{\pi_g^O}
\newcommand{\gfpi}{\pi_g^F}

\newcommand{\enumr}{\begin{enumerate}[label=\roman{*})]}
\newcommand{\enumR}{\begin{enumerate}[label=\Roman{*})]}
\newcommand{\enuma}{\begin{enumerate}[label=\alph{*})]}
\renewcommand{\leq}{\leqslant}
\renewcommand{\le}{\leqslant}  

\renewcommand{\geq}{\geqslant}
\renewcommand{\ge}{\geqslant}

\newcommand{\uth}{\underline{\theta}}
\newcommand{\huth}{\hat{\underline{\theta}}}
\newcommand{\uthi}{\underline{\theta}\,_i}

\newcommand{\uuth}{\underline{\underline{\theta}}}

\newcommand{\uTh}{\underline{\Theta}}
\newcommand{\uThi}{\underline{\Theta}\,_i}
\newcommand{\uuTh}{\underline{\underline{\Theta}}}
\newcommand{\ism}{$\mathbf  g$-$\mathbf I\mathbf S\mathbf M$ index} 
\newcommand{\gfp}{$\mathbf  g$-Forcing}

 \newcommand{\ase}{ $\exists$  (a.s.) \ }

\title{
Minimal-Exploration Allocation Policies: \\
Asymptotic, Almost Sure, Arbitrarily Slow Growing Regret
 }
 
\author{{\bf 
Wesley Cowan} \\
   Department of Mathematics, Rutgers University\\
110 Frelinghuysen Rd., Piscataway, NJ 08854 
\and {\bf Michael N. Katehakis}\\
Department of Management Science and Information Systems\\
 100 Rockafeller Road, Piscataway, NJ 08854, USA
}
\begin{document}
\maketitle
\begin{abstract}
The purpose of this paper is to provide further understanding  into the structure of the sequential allocation (``stochastic multi-armed bandit'',  
or MAB) problem by establishing probability one finite horizon bounds and convergence rates for the sample (or ``pseudo'') regret associated with two simple classes of allocation policies $\pi$. 

For any slowly increasing function $g$, subject to mild regularity constraints, we construct two policies (the $g$-Forcing, and the $g$-Inflated Sample Mean) that achieve a measure of regret of order $  O(g(n))$ almost surely as $n \to \infty$, bound from above and below. Additionally, almost sure upper and lower bounds on the remainder term are established. In the constructions herein, the function $g$ effectively controls the ``exploration'' of the classical ``exploration/exploitation'' tradeoff. 
 
 \end{abstract}

\ \\
{\bf Keywords:}
  Forcing  Actions, Inflated Sample Means, Multi-armed Bandits, Sequential Allocation,  Online Learning

\section{Introduction and Summary}\label{sec:intro}

The basic problem involves  sampling sequentially from a finite number of $K \ge 2$    populations  or ``bandits,'' 
where each population $i$  is specified by a sequence of real-valued i.i.d. random variables, $\{ X^i_k \}_{k \geq 1}$,   with $X^i_k$ representing  the reward received the $k^{th}$ time 
population  $i$ is sampled. 
The distributions $F_i$ of the $X^i_k$ are taken to be unknown;  they belong to some collection of distributions  $\mathcal{F}$.  We restrict $\mathcal{F}$   in two ways: 

The first, that each population $i$ has some finite mean $\mu_i = \mathbb{E}[ X^i_k ] =\int_{-\infty}^{+\infty} xdF_i(x)< \infty$ -  unknown to the controller.  The purpose of this assumption is to establish for each population $i$ the Strong Law of Large Numbers (SLLN),
\begin{equation}\label{eqn:stln}
\mathbb{P} \left( \lim_k \bar X^i_k = \mu_i \right) = 1. 
\end{equation}

Second, we assert that each population  has finite variance $\sigma_i^2 = \text{Var}(X^i_k) < \infty$. The purpose of this assumption is to establish for each population $i$ the Law of the Iterated Logarithm (LIL),
\begin{equation}\label{eqn:lil}
\mathbb{P} \left( \limsup_k \pm \frac{\bar{X}^i_{k} - \mu_i}{ \sqrt{ \ln \ln k / k } } = \sigma_i \sqrt{2} \right) = 1.
\end{equation}

It will emerge that the important distribution properties for the populations are not the i.i.d. structure, but rather Eqs. \eqref{eqn:stln}, \eqref{eqn:lil} alone. This allows for some relaxation of assumptions, as discussed in Section \ref{sec:relax}. In fact, the LIL (and therefore the assumption of finite variances) is only really required for the derivation of the regret remainder term bounds in the results to follow - the primary asymptotic results depend solely on the SLLN.

 Additionally, we define $\mu^* = \max_i \mu_i$, and we take the optimal bandit to be unique - that is, there is a unique $i^*$ such that $\mu_{i^*} = \mu^*$.
It is convenient to define the bandit discrepancies $\{ \Delta_i \}$ as $\Delta_i = \mu^* - \mu_i \geq 0$.

 For any    adaptive policy  $\pi$,  let 
 $\pi(t) = i$ indicate the event that population $i$ is sampled at time $t$,  and  let 
 $T^i_\pi(n) = \sum_{t = 1}^n \mathbf {1}_{\pi(t) = i}$ 
 denote the number of times $i$  has been sampled during  periods $t = 1, 2, \ldots, n$,  under policy $\pi$; for convenience we define $T^i_\pi(0) = 0$ for all $i, \pi$.
One is typically interested in  maximizing in some well defined sense 
   the  sum    of the first $n$ outcomes
$S_\pi(n)  = 
  \sum_{i=1}^K     \sum_{k = 1}^{T^i_\pi(n)}  X^i_k ,$ achieved by   an adaptive  policy $\pi  .$ 
To this end we note
that if   the   controller had complete information (i.e., knew the distributions of the $X^i_k$, for each $i$), she would at every round activate the ``optimal'' bandit $i^*$.  Natural measures of the  loss due to this ignorance 
of the distributions, are the quantities below:
 
%
%
\begin{align}
\label{eqn:pseudo-regret}
\tilde{R}_\pi(n) 
& = n \mu^* - \sum_{i = 1}^K \mu_i  T^i_\pi(n) 
 = \sum_{i = 1}^K \Delta_i T^i_\pi(n) ,\\   
   \label{eqn:regret}
     R_\pi(n)  & = n \mu^* -\mathbb{E}\left[S_\pi(n)\right]
= \sum_{i = 1}^K \Delta_i \mathbb{E}\left[ T^i_\pi(n) \right]. 
\end{align}

The
functions $\tilde{R}_\pi(n) $,   $R_\pi(n)$ have been called in the literature  pseudo-reget,  and regret;  for notational simplicity their  dependence  
 on the unknown  distributions is usually  suppressed.  

The motivation for considering minimizing alternative regret measures to $R_\pi(n)$ is that while the investigator might be pleased to know that the policy she is utilizing has minimal \emph{expected} regret, she might reasonably be more interested in behavior of the policy on the specific sample-path she is currently exploring rather than aggregate behavior over the entire probability space.
At an extreme end of this would be a result minimizing regret or pseudo-regret surely (sample-path-wise) or almost surely (with full probability), guaranteeing a sense of optimality independent of outcome. We offer an asymptotic result of this type here  in Theorem 2.

Note that $    \mathbb{E} [\tilde{R}_\pi(n)] = R_\pi(n) $, and 
  ``good policies'' are those  that achieve a small rate of increase for one of the above regret functions. 
 Further relationships and forms of   pseudo-regret  are explored in \cc{bubeck2012regret}, e.g., 
the ``sample regret''    $ R'_\pi(n)   = n \mu^* - S_\pi(n)
 = n \mu^* -    \sum_{i=1}^K     \sum_{k = 1}^{T^i_\pi(n)}  X^i_k .$ 
 We find the pseudo-reget
$ \tilde{R}_\pi(n) 
  = n \mu^* - \sum_{i = 1}^K \mu_i T^i_\pi(n) 
$
in some sense more philosophically satisfying to consider than sample regret, for the reason that - given her ignorance and the inherent randomness - the controller cannot reasonably regret the {\sl specific reward} gained or lost from an activation of a bandit, as in $  R'_\pi(n) .$ She can only reasonably regret {\sl the decision to activate that specific bandit}, which is captured by $\tilde{R}_\pi(n)$'s dependence on the $T^i_\pi(n)$s alone.

Thus, we  are particularly interested in high probability or guaranteed (almost sure) asymptotic bounds on the growth of the pseudo-regret as $n \to \infty$. 
The main result of this paper is Theorem \ref{thm:exist}  which establishes, by two examples,  that for any arbitrarily (slowly) increasing function $g(n)$, e.g.,  $g(n) = \ln \ln \ldots \ln n$, that satisfies mild regularity conditions  
 there exist   ``$g$-good policies'' ${\pi_g}$\,. The later policies are such that the following is true 
$$\tilde{R}_{\pi_g}(n)=C_{\pi_g}( \{ F_i \} ) g(n) + o(g(n)) , \mbox{ as $n\to \infty$} $$ 
(i.e., $\tilde{R}_{\pi_g}(n) = O(g(n)), \ \mbox{(a.s),  as \ } n \to \infty  $)   for every set of bandit distributions $\{ F_i \} \subset \mathcal{F}$, 
 for some positive 
finite constant $C_{\pi_g}( \{ F_i \}) $.

The results presented here are in fact intuitive, in the following way: it will be shown that in the \gfp\ and \ism\ policies, the function $g$ essentially sets the investigator's willingness to explore and experiment with bandits that do not currently (based on available data) seem to have the highest mean. Even if the controller explores very slowly (i.e., she chose a very slow growing $g$), {\sl as long as she explores long enough} she will eventually develop accurate estimates of the means for each bandit, and incur very little regret (or pseudo-regret) past that point. We note here that, for the most part, we   do not recommend the actual implementation or use of these policies. The cost of this guaranteed asymptotic behavior is that (depending on $g$ and the bandit specifics), slow pseudo-regret growth is only achieved on impractically large time-scales. We find it interesting, however, that such growth can be guaranteed - independent of the specifics of the bandits! - with as weak assumptions as the Strong Law of Large Numbers. This makes these results fairly broad. Additionally, the \gfp\ and \ism\ policies individually capture elements present in many other popular policies, and are suggestive of the almost sure asymptotical behavior of these policies. One takeaway from this is, perhaps, to emphasize that asymptotic behavior by itself is little basis for thinking of a policy as ``good''. As essentially any asymptotic behavior is possible (through the choice of $g$), any useful qualification of a policy must consider not only the asymptotic behavior, but also the timescales over which it is practically achieved.

In the remainder of the paper, we define what it means for a policy to be $g$-good (Definition \ref{def:1}), and establish the existence of $g$-good policies (Theorem \ref{thm:exist}) for any $g$ satisfying mild regularity conditions. The proof is by example, through the construction of \gfp\ and \ism\ policies that satisfy its claim. Further, bounds on the corresponding order constants of pseudo-regret growth are established for each policy (Theorems \ref{thm:gfg} and \ref{thm:uindex-unique}), as well as bounds on the asymptotic remainder terms (Theorems \ref{thm:gfg-strong} and 5 \ref{thm:uindex-unique-remainder}), bounding the remainder from both above and below. We view the proofs of the asymptotic lower bounds, as well as the derivation of the remainder terms via a sort of ÔbootstrappingÕ on the earlier order results, as particularly interesting.

In the attempt to generalize some of these results for the \ism\ policy, an interesting effect and seeming ``phase change'' in the resulting dynamics was discovered. Specifically, as  discussed in Remark 2, when there are multiple optimal bandits,  for $g$ of order greater than $\sqrt{ n \ln \ln n }$ all optimal bandits are sampled roughly equally often, while for $g$ of order less than $\sqrt{n \ln \ln n}$, the \ism\ policy tends to fix on a single optimal bandit, sampling the other optimal bandits much more rarely in comparison.

\section{Related Literature}\label{sec:intro}

\cc{Rb52} first analyzed  the problem of maximizing asymptotically the expected value of 
  the  sum     
$S_\pi(n) .$ Using 
  only the assumption of the Strong Law of Large Numbers for $\mathcal{F}, $ for $K=2$. He constructed  a modified (outside two sparse sequences of  forced choices)  ``play the winner'' (greedy)  policy, $\pi_R$,  such that  with probability one, as $n\to \infty ,$  $S_{\pi_R}(n)/n \to  \mu^* $.
 From this  he was able to claim, using the  uniformly integrability  property 
 for the case of Bernoulli bandits that 
 \begin{equation} \label{eqn:regret1}
  R_{\pi_R}(n)   = o(n), \ \mbox{as $n\to \infty $}.
 \end{equation}

 \cc{lai85}  considered  the case in which the collection of distributions $\mathcal{F}$     to  consist  of  univariate density functions $f(x;\theta_i)$ with respect to some measure $\nu_i ,$ where $f( . ; .)$ is known and the unknown  scalar parameter $\theta_i $ is in some known set $    \Theta .$ 
 Let $\mu_i=\mu (\theta_i)= \mathbb{E}[X^i_1],$  
   $\mu^*=\max_i\{\mu (\theta_i)\}=\mu(\theta^*), $   $\Delta_i(\theta_i)=\mu(\theta^*)-\mu (\theta_i)$, and let
   $\mathbb{I} (\theta ||\theta') = \int_{-\infty}^\infty \! \ln 
   \frac{f(x;\theta)}{f(x;\theta')} f(x;\theta) \, \mathrm{d}v(x)  $
   denote the Kullback - Leibler divergence between $f(x;\theta)$ and $f(x;\theta').$ They established, under mild regularity conditions ((1.6), (1.7) and (1.9) therein),  that if one requires a policy 
to have  a regret that increases at  slower than linear rate: 
\begin{equation} \label{eqn:regret2}
  R_\pi(n)   = o(n^\alpha), \ \mbox{$\forall \alpha>0$, as $n\to \infty $}, 
\mbox{ $\forall \left\{ \theta_i \right\} \subset \Theta$} , 
\end{equation}
then $\pi$  must sample among populations in such as way that its regret satisfies 
\begin{equation}\label{en:lower-bound-lr}
\liminf_n \frac{ R_\pi(n) }{ \ln n } \geq M_{\text{LR}}( \theta_1, \ldots,
\theta_K),  \ 
\mbox{ $\forall \left\{ \theta_i \right\} \subset \Theta$} , 
\end{equation} 
where  
 $$M_{\text{LR}}( \theta_1, \ldots,
\theta_K)=  \sum_{i:\mu(\theta_i) \neq \mu^*}   \Delta_i(\theta_i)/\mathbb{I} (\theta_i ||\theta^*).$$

  \cc{bkmab96}
extended and simplified  the above  work for the case in which   the collection of distribution $\mathcal{F}$ is specified by a known function 
 $f(x;\uthi)$ that may depend on 
an unknown  vector parameter $\uthi \in \uThi ,$   as follows. 
Let $\uuth :=(\uth\,_1,\ldots,\uth\,_K)\in\uuTh =\uTh\,_1\times\cdots\times\uTh\,_K,$ \ $\mu^*=\mu(\uuth^*)=\max_i\{\mu(\uthi)\},$ $\Delta_i(\uth\,_i)=\mu^*-\mu(\uthi)$. 
%
They showed, under certain  regularity conditions (part 1 of Theorem 1, therein) that if a policy satisfied 
\eqqref{eqn:regret2},  $\forall \uuth\in\uuTh,$ 
 then it must sample among populations in such as way that its regret satisfies: 
\begin{equation}\label{en:lower-bound-bk}
\liminf_n \frac{ R_\pi(n) }{ \ln n } \geq M_{\text{BK}}(\uuth),  \ 
\mbox{ $\forall \uuth\in\uuTh$} , 
\end{equation} 
where  
%
\begin{equation}\label{en:lb-bk}
M_{\text{BK}}(\uuth)=  \sum_{i\in B(\uuth)}   
 \Delta_i(\uthi)/\inf_{\uthi'}\{ \mathbb{I}(\uthi,\uthi')  \ : \  \mu(\uthi')>\mu(\uuth^*) \}.
\end{equation}

Further, under  certain  regularity conditions (cf. conditions  ``A1-A3'' therein)  regarding    
 the  estimates  $ {\huth}\,_i$ $ =\huth^n_i(X^i_1,\ldots,X^i_{T_{\pi}(n)})$ 
 of the parameters   $ {\uth}\,_i ,$ $f(.;.)$ and $\uTh_i$, 
 they showed that policies which, after 
 taking some small number of samples from each population, always choose the population $\pi^{0}(n)$  with the largest value of the population dependent index:
\begin{equation}\label{en:ism-bk}
 u_i(\huth^{\,n}_{\,i})=\sup_{\uthi'\in\uTh_i}\left\{\mu(\uthi ') \ : \  \mathbb{I}(\huth^{\,n}_{\,i},\uthi') < \frac{\ln n +o(\ln n)}{T_{\pi^0}^i(n)}\right\} .
\end{equation} 
are asymptotically efficient (or optimal), i.e.,  
\begin{equation}\label{en:ub-bk}
\limsup_n \frac{ R_{\pi^0}(n) }{ \ln n } \leq M_{\text{BK}}( \uth_1, \ldots,
\uth_K),  \ 
\mbox{ $\forall \uuth\in\uuTh$}.
\end{equation}

The index policy $\pi^0$ above, was a simplification of a UCB type policy first introduced  in  \cc{lai85} that  utilized forced actions.  Policies that satisfy the requirements of \eqqref{eqn:regret1}, 
 \eqqref{eqn:regret2}, and  
 \eqqref{en:ub-bk} were respectively called 
   \textbf{\textit{uniformly  consistent}} (UC), \textbf{\textit{uniformly fast convergent}} (UF), and 
\textbf{\textit{uniformly maximal convergence rate}} (UM) or simply   \textbf{\textit{asymptotically optimal}} (or asymptotically efficient).
%
%
The lower bound of \eqqref{en:lb-bk}  provides a baseline for comparison of the quality of policies and together with 
\eqqref{en:ub-bk} and \eqqref{en:lower-bound-bk} provide an alternative way to state the asymptotic  optimality of a policy $\pi^0$ as:
\begin{equation}\label{en:ao}
 R_{\pi^0}(n)   = M_{\text{BK}}(\uuth)  \ln n +o(\ln n) \, ,  \ 
\mbox{ $\forall \uuth\in\uuTh$}.
\end{equation}  
Policies that achieve this minimal 
asymptotic growth rate have been derived for specific parametric models in 
\cc{lai85},   \cc{bkmab96},  
\cc{honda2011asymptotically},   \cc{honda2010},  \cc{honda13}, \cc{chk2015} and references therein. 
In general it is not always easy to obtain such optimal polices, thus, policies that satisfy the
  less strict requirement  of 
\eqqref{eqn:regret2},  $\forall \uuth\in\uuTh,$ have been constructed, cf. 
\cc{Auer02b}, \cc{audibert2009exploration}, \cc{bubeck2012regret} and references therein. Such policies usually bound the regret as follows:
 \begin{equation}\label{en:a-bound}
 R_{\pi}(n) \le  M^0(\uuth) \ln n + M^1(\uuth) ,
\mbox{ for all $n$ and all  $\uuth$} , 
\end{equation}
where  $M^0(\uuth)$ is, often much, bigger than $M_{\text{BK}}(\uuth) , $ for all   $\uuth$. 

The results presented herein   can seem surprising, and it may appear to contradict (at least for $g(n)=\ln n$) the   classical   lower bound $M_{\text{BK}}(\uuth)$  of   $R_{\pi}(n) /\ln n  $ for UF policies $\pi$. For example,  if we take $F_i$ to be the normal 
distribution with unknown mean $\mu_i$ and unknown  variance $\sigma_i^2$, we have   
  for any UF policy $\pi$:
 $$\lim_n \frac{\mathbb{E}[\tilde{R}_{\pi}(n)]}{\ln n} \ge 
 \mathbb{M}_{\text{BK}}(\underline{\mu}, \underline{\sigma}^2)
 = \sum_{i:\mu_i \neq \mu^*} \frac{ 2 \Delta_i }{ \ln \left( 1 + \frac{ \Delta_i^2 }{\sigma_i^2} \right)} \ .$$
 

On the other hand we establish in the sequel that:
\begin{equation}
\begin{split}
\lim_n \frac{\tilde{R}_{\pi^F_g}(n)}{g(n)} & = C_{\pi^F_g}(\{ F_i \})  =\sum_{i: \mu_i \neq \mu^* } \Delta_i\ \mbox{ (a.s.)},\\
 \lim_n \frac{\tilde{R}_{\ismi} (n)}{g(n)} & = C_{\pi^O_g}(\{ F_i \})= K-1 \ \mbox{ (a.s.)}.
 \end{split}
 \end{equation}
However, no such contradiction exists: $M_{\text{BK}}(\uuth)$ limits the  
$\lim_n  \mathbb{E}[\tilde{R}_{\pi}(n)]/\ln n $  of a UF policy from below. In such contexts that $\pi^F_g$ or $\pi^O_g$ are UF, if such contexts exist, the above constants will be bounded from below by $M_{\text{BK}}(\uuth)$. In such contexts that $\pi^F_g$ or $\pi^O_g$ are not UF, the bound does not apply. In such instances, we may in fact conclude from the results presented herein, and   standard results relating modes of convergence, that for the  policies  constructed here, for $g(n) = O( \ln n )$, the sequences of random variables $\tilde{R}_{\pi^F_g}(n) / g(n)$,  $\tilde{R}_{\pi^O_g}(n) / g(n)$ are not uniformly integrable. An example as to how 
this can occur is given via the proof of Theorem 2 of \cc{chk2015}, where with 
 a non-trivial probability, non-representative initial sampling of each bandit biases expected future activations of sub-optimal bandits super-logarithmically. This effect does not influence the long term almost sure behavior of these policies.

\section{Main Theorems}\label{sec:framework}

We 
characterize a policy by the rate of growth of its pseudo-regret function
$\tilde{R}_\pi(n) $  with $n$ in the following way. 

\begin{definition}\label{def:1}
For a function $g(n)$,  a policy $\pi$ is  $g$-{\sl good} if for every set of bandit distributions $\{ F_i \} \subset \mathcal{F}$, there exists a constant $C_\pi( \{ F_i \} ) < \infty$ such that
\begin{equation}\label{eq:regr}
  \limsup_n \frac{\tilde{R}_\pi(n)}{g(n)} \leq C_\pi(\{ F_i \}) \  (a.s) \ as \  n \to \infty .
\end{equation}
\end{definition}

{\bf Remark 1:} Essentially, a policy is $g$-good if $\tilde{R}_\pi(n) \leq O(g(n))$  (a.s),    $n \to \infty.$ Trivially, policies exist that are $n$-good (i.e., $\tilde{R}_\pi(n) \leq O(n) \mbox{ (a.s.)}$), for example any policy that samples all populations at constant rate $1/K$.

We next state  the following theorem:

\begin{theorem}\label{thm:exist}
For $g$, an unbounded, positive, increasing, concave, differentiable, sub-linear function, there exist $g$-good policies.
\end{theorem}

The proof of this theorem  is given  by example with  Theorems \ref{thm:gfg}, \ref{thm:uindex-unique}, which
demonstrate two $g$-good policies: the \gfp \ and the \ism \ policies. 

We note 
that in the sequel it will be assumed that any $g$ considered is an  {\sl unbounded, positive, increasing, concave, differentiable, sub-linear} function.

\subsection{A Class of $\mathbf g$-Forcing Policies}\label{sec:forcing}

Let $g$ be as hypothesized in Theorem \ref{thm:exist}. We define a \gfp \ policy $\pi^F_g$ in the following way:

\vspace{0.1cm}
\noindent
\begin{tcolorbox}[colback=blue!1, arc=1pt, width=.99\linewidth]
\textbf{\gfp \ policy:} 
A policy $\pi^F_g$ that  first samples  each bandit once,   then for $t \geq K$,

\begin{equation}\label{eqn:pi-g-1}
\pi^F_g(t+1) = 
\begin{cases} \arg\max_i \bar X^i_{T^i_{\pi^F_g}(t)} &\mbox{if }  \min_i T^i_{\pi^F_g}(t)  \ge  g(t),   \\ 
\arg\min_i T^i_{\pi^F_g}(t)  & \mbox{else}. \end{cases}
\end{equation}
\end{tcolorbox}
\noindent

Briefly, at any time, if any population has been sampled fewer than $g(t)$ times, sample it. Otherwise, sample from the population with the current highest sample mean. Ties are broken either uniformly at random, or at the discretion of the investigator. In this way, $g$ can be seen as determining the rate of exploration of currently sub-optimal bandits. This can be viewed as a variant on the policy $\pi_R$ considered in \cc{Rb52}.

It is convenient to define the following constant,
\begin{equation}
S_\Delta = \sum_{i: \mu_i \neq \mu^* } \Delta_i.
\end{equation}
The value $S_\Delta$ in some sense represents the pseudo-regret incurred each time the sub-optimal bandits are all activated once. The next result states that \gfp \ policies satisfy the conditions of Theorem \ref{thm:exist}.

\begin{theorem}\label{thm:gfg}
For a policy $\pi^F_g$ as in \eqref{eqn:pi-g-1}, $\pi^F_g$ is $g$-good, and
\begin{equation}
\mathbb{P} \left( \lim_n \frac{\tilde{R}_{\pi^F_g}(n)}{g(n)} = S_\Delta  \right) = 1.
\end{equation}
\end{theorem}


The above theorem can be strengthened in the following way, bounding the asymptotic remainder terms almost surely:

\begin{theorem}\label{thm:gfg-remainder} For a policy $\pi^F_g$ as in \eqref{eqn:pi-g-1}, the 
following are true:
\begin{equation}
\mathbb{P} \left( \limsup_n\ \left(\tilde{R}_{\pi^F_g}(n) - S_\Delta g(n) \right) \leq S_\Delta  \right) = 1,
\end{equation}
and
\begin{equation}
\mathbb{P} \left( \liminf_n\ \left( \tilde{R}_{\pi^F_g}(n) - S_\Delta g(n) \right) \geq  0  \right) = 1.
\end{equation}
\end{theorem}

\begin{proof}[Theorems \ref{thm:gfg} and \ref{thm:gfg-remainder}]
Theorems \ref{thm:gfg}, \ref{thm:gfg-remainder} follow immediately from the following proposition, the proof of which is given in Appendix \ref{apx:gfg-strong}:

\begin{proposition}\label{thm:gfg-strong}
For  policy $\pi^F_g$ as in \eqref{eqn:pi-g-1}, the following is true: For every $\epsilon > 0$, almost surely there exists a $N_\epsilon < \infty$ such that, for all $n \geq N_\epsilon$,
\begin{equation}
g(n) S_\Delta - \epsilon \leq \tilde{R}_{\pi^F_g}(n) \leq \ceil{ g(n) } S_\Delta.
\end{equation}
\end{proposition}

Using the above relation to bound first the limits as $n \to \infty$ of $\tilde{R}_{\pi^F_g}(n) / g(n)$, then $\tilde{R}_{\pi^F_g}(n) - S_\Delta g(n)$ (observing that $\ceil{ g(n) } - g(n) \leq 1$), give the desired results.
\end{proof}


Proposition \ref{thm:gfg-strong} is considerably stronger than Theorems \ref{thm:gfg}, \ref{thm:gfg-remainder}. However, it somewhat obscures the true nature of what is going on: for sufficiently large $n$, almost surely, sub-optimal bandits $(i: \mu_i \neq \mu^*)$ are {\sl only} activated during the ``forcing'' phase of the policy, when some activations are below $g$. As a result, since $g$ increases slowly (e.g. is sub-linearly), for large $n$, $T^i_{\pi^F_g}(n) = \ceil{ g(n) }$ - except for a discrepancy that occurs, for a brief stretch $(< K)$ of activations, whenever $g$ surpasses the next integer threshold. At this point, the policy raises the activations of each sub-optimal bandit, restoring the previous equality. Hence, in fact, equality holds in Proposition \ref{thm:gfg-strong} $(\tilde{R}_{\pi^F_g}(n) = \ceil{ g(n) } S_\Delta)$ for most large $n$. Discrepancy occurs increasingly rarely with $n$, based on the hypotheses on $g$. If, additionally, the controller specifies a deterministic scheme for tie-breaking, pseudo-regret may be determined explicitly for all sufficiently large $n$. Leaving ties to the discretion of the controller, Proposition \ref{thm:gfg-strong} is as strong a statement as can be made.

\subsection{A Class of g-Index Policies}\label{sec:indices}

In this section, we consider an index policy related to the classical ''UCB'' index policies. Let $g$ be as hypothesized. For each $i$, define an \textit{index} on $(j,k) \in \mathbb{Z}_+^2$,
\begin{equation}
u_i(j, k) = \bar X^i_k + \frac{g(j)}{k}.
\end{equation}

\vspace{0.1cm}
\noindent
\begin{tcolorbox}[colback=blue!1, arc=1pt, width=.99\linewidth]
\textbf{\ism \ policy:} 
A policy $\ismi$ that  first samples  each bandit once,   then for $t \geq K$,
\begin{equation}\label{eqn:pi-g-2}
\ismi (t+1) = \arg\max_i u_i(t, T^i_{\ismi}(t)) = \arg\max_i \left( \bar  X^i_{ T^i_{\ismi}(t) } + \frac{ g(t) }{ T^i_{\ismi}(t) } \right).
\end{equation}
\end{tcolorbox}

Briefly, at any time, the sample means of each bandit are ``inflated'' by the $g(t)/T^i_{\pi^O_g}(t)$ term, and the policy always activates the bandit with the largest inflated sample mean. When unsampled, a bandit's inflated sample mean increases essentially at rate $g$, hence $g$ drives the rate of exploration of current sub-optimal bandits. While this policy is inspired by more traditional ''Upper Confidence Bound'' policies, we refer to this as an Inflated Sample Mean policy, as it has no deliberate connection to confidence bounds. 

More general index policies of this type could also be considered, for instance based on an index $\bar{X}^i_k + H_i\left( g(j) / k \right)$ where $H_i$ is some positive, increasing function of its argument. This is more in line with the common UCB policies, which frequently have inflation terms of the form $O\left( \sqrt{ \ln n / T^i_\pi(n) } \right)$ (though this is hardly necessary, c.f. \cc{chk2015}) with $\ln n$ serving the ``exploration-driving'' role of $g$. However, introducing this extra $H_i$ function does not influence the order of the growth of pseudo-regret, it simply changes the relevant order constants, at the cost of complicating the analysis.

Theorem \ref{thm:uindex-unique} below shows that a \ism \ policy satisfies the conditions of Theorem \ref{thm:exist}, and gives the minimal order constant $C_{\pi^O_g}$ for this policy.

\begin{theorem}\label{thm:uindex-unique}
For a policy $\ismi$ as in \eqref{eqn:pi-g-2}, if the optimal bandit is unique,
\begin{equation}
\mathbb{P} \left( \lim_n \frac{\tilde{R}_{\ismi} (n)}{g(n)} = K-1  \right) = 1.
\end{equation}
\end{theorem}

The proof of this theorem depends on the following propositions, the proofs of which are given in Appendix \ref{apx:cub}. Interestingly, these results (and therefore Theorem \ref{thm:uindex-unique}) depend only on the assumption of the SLLN, not the LIL.

\begin{proposition}\label{prop1}
For each sub-optimal $i$, $\forall  \epsilon \in(0 , \Delta_i/2)$, \ase a finite constant $C^i_\epsilon$ such that for $n \geq K$,
\begin{equation}
T^i_{\ismi}(n) \leq \frac{ g(n) }{ \Delta_i - 2\epsilon } + C^i_\epsilon.
\end{equation}
\end{proposition}

\begin{proposition}\label{prop2}
For each sub-optimal $i \neq i^*$,    
 $\forall \epsilon \in(0,\min_{j \neq i^*} \Delta_j/2)$, \ase   some finite $N'$ such that for $n \geq N'$,
\begin{equation}
\frac{g(n)}{(1 + \epsilon)(\Delta_i + 2\epsilon) + 2 \epsilon} \leq T^i_{\ismi}(n).
\end{equation}
\end{proposition}

\begin{proof}[Theorem \ref{thm:uindex-unique}]
For each sub-optimal bandit $i$, as an application of Props. \ref{prop1}, \ref{prop2}, taking the limit of 
$T^i_{\ismi}(n)/g(n)$ first as $n \rightarrow \infty$, then as $\epsilon \rightarrow 0$, gives
 $\lim_n T^i_{\ismi}(n)/g(n) = 1/\Delta_i$, almost surely. The theorem then follows similarly, from the definition of pseudo-regret, Eq. \eqref{eqn:pseudo-regret}.
\end{proof}

{\bf Remark 2:} In the case that the optimal bandit is not unique, it happens that Prop. \ref{prop1} still holds. It can be shown then that $\pi^O_g$ remains $g$-good in this case, and has a limiting order constant of at most $K - K^*$ ($K^*$ as the number of optimal bandits). We leave as an open question, however, that of producing a Prop. \ref{prop2}-type lower bound and the verification of $K - K^*$ as the minimal order constant. The proof of Prop. \ref{prop2} for $K^* = 1$ depends on establishing a lower bound on the activations of the unique optimal bandit: in short, at time $n$, since the sub-optimal bandits are activated at most $O(g(n))$ times (which holds independent of $K^*$), it follows from its uniqueness that the optimal bandit is activated {\sl at least} $n - O(g(n))$ times. If, however, $ K^* > 1$ and the optimal bandit is not unique, while the optimal bandits must have been activated at least $n - O(g(n))$ {\sl in total} at time $n$, and the distribution of these activations among the optimal bandits is hard to pin down. Simple simulations seem to indicate a sort of ``phase change'', in that for $g$ of order greater than $\sqrt{ n \ln \ln n }$ all optimal bandits are sampled roughly equally often, while for $g$ of order less than $\sqrt{n \ln \ln n}$, the policy tends to fix on a single optimal bandit, sampling the other optimal bandits much more rarely in comparison.

We offer the following as a potential explanation of this observed effect (and justification of the difficult to observe $\ln \ln n$ term): Let us hypothesize, for the moment, that under any circumstances, the optimal bandits are activated linearly with time, that is for any optimal $i^*$, $T^{i^*}_{\pi^O_g}(n) = O(n)$, with the order coefficient depending on the specifics of that bandit. Under policy $\pi^O_g$, activations are governed by a comparison of indices. We consider then the fluctuations in value of the two terms of the index, the sample mean $\bar{X}^{i^*}_{T^{i^*}_{\pi^O_g}(n)}$ and the inflation term $g(n)/T^{i^*}_{\pi^O_g}(n)$. Under the assumption the optimal bandits are activated linearly, and reasonable assumptions on the bandit distributions (to grant the Law of the Iterated Logarithm), the fluctuations in the sample mean over time will be of order $O( \sqrt{\ln \ln n / n})$. The fluctuations in the inflation term will be of order $O(g(n)/n)$.  It would seem to follow then that for $g$ of order less than $O(\sqrt{ n \ln \ln n }),$ when comparing indices of optimal bandits, the sample mean is the dominant contribution to the index, while for $g$ of order greater than $O(\sqrt{n \ln \ln n})$, the inflation term is the dominant contribution to the index. When the inflation term dominates, among the optimal bandits an ``activate according to the largest index'' policy essentially reduces to a ``activate according to the smallest number of activations'' policy, which leads to equalization and all optimal bandits being activated roughly equally often. When the sample mean dominates, among the optimal bandits an ``activate according to the largest index'' policy essentially reduces to an ``activate according to the highest sample mean'' or ``play the winner'' policy, which leads to the policy fixing on certain bandits for long periods.

This explanation would additionally suggest that on one side of the phase change, when the inflation term dominates, the only properties of the optimal bandits that matter for the dynamics of the problem are their means, that they all have the optimal mean $\mu^*$. But on the other side of the phase change, when the sample mean dominates, other properties such as the variances $\{ \sigma^2_i \}$ influence the dynamics, through the Law of the Iterated Logarithm.
However at this point in time, this remains, while interesting, speculative.

Based on the above results, we have the following result: For each $i \neq i^* ,$  $\forall \epsilon > 0$, \ase some finite $N_\epsilon$ such that for $n \geq N_\epsilon$, 
\begin{equation}
\frac{1-\epsilon}{\Delta_i} g(n) \leq T^i_{\ismi}(n) \leq \frac{1+\epsilon}{\Delta_i} g(n).
\end{equation}
Similarly, for the optimal bandit $i^*$,
\begin{equation}
n - (1 + \epsilon) \sum_{i \neq i^*} \frac{1}{\Delta_i} g(n) \leq T^{i^*}_{\ismi}(n) \leq n - (1 - \epsilon) \sum_{i \neq i^*} \frac{1}{\Delta_i} g(n).
\end{equation}
It follows trivially from these that each bandit is activated infinitely often, i.e., almost surely $\{ T^i_{\ismi}(n) \}_{n \geq 1}$ 
is equivalent to the sequence $\{0, 1, \ldots\}$, though with some (finite) stretches of term repetition. It follows then, applying the LIL that
\begin{equation}
\mathbb{P} \left( \limsup_n \pm \frac{\bar X^i_{T^i_{\ismi}(n)} - \mu_i}{ \sqrt{ \ln \ln T^i_{\ismi}(n) / T^i_{\pi^O_g}(n) } } = \sigma_i \sqrt{2} \right) = 1.
\end{equation}
This provides greater control over the sample mean of each bandit than what the Strong Law of Large Numbers alone allows, and allows the results of the previous asymptotic results to be strengthened, as in the following theorem.

\begin{theorem}\label{thm:uindex-unique-remainder} For a policy $\pi^O_g$ as in \eqref{eqn:pi-g-2}, then the 
following are true:

{a)}  if    $g(n) = o( n / \ln \ln n )$,
\begin{equation}
\mathbb{P} \left( \limsup_n \frac{\tilde{R}_{\pi^O_g}(n) - (K-1)g(n)}{\sqrt{g(n) \ln \ln g(n)} } \leq 2 \sqrt{2} \sum_{i \neq i^*} \frac{\sigma_i}{\sqrt{\Delta_i}}  \right) = 1,
\end{equation}
{b)} if $g(n) = o( n^{2/3} )$,
\begin{equation}
\mathbb{P} \left( \liminf_n \frac{\tilde{R}_{\pi^O_g}(n) - (K-1)g(n)}{\sqrt{g(n) \ln \ln g(n)} } \geq  -3  \sqrt{2} \sum_{i \neq i^*} \frac{\sigma_i}{  \sqrt{ \Delta_i } }  \right) = 1.
\end{equation}
\end{theorem}

In short, we have that for a \ism\ policy $\ismi$, $$\tilde{R}_{\pi^O_g}(n) = (K-1)g(n) + O\left( \sqrt{ g(n) \ln \ln g(n) } \right).$$  It should be observed that, unlike previous results, this theorem is somewhat restrictive in its allowed $g$. However, since the focus is traditionally on logarithmic regret, i.e., $g(n) = O(\ln n)$, it is clear that the above restrictions are nothing serious.

This theorem follows trivially from the following refinements of Props. \ref{prop1}, \ref{prop2}, and the definition of pseudo-regret, Eq. \eqref{eqn:pseudo-regret}. Their proofs are given in Appendix \ref{apx:lil}.

\begin{proposition}\label{prop3}
If $g(n) = o(n / \ln \ln n )$, for each sub-optimal $i \neq i^*$, the following holds almost surely:
\begin{equation}
\limsup_n \frac{ \Delta_i T^i_{\ismi}(n) - g(n)}{ \sqrt{ g(n) \ln \ln g(n) }}  \leq \frac{2 \sigma_i \sqrt{2} }{\sqrt{\Delta_i}}.
\end{equation}
\end{proposition}

\begin{proposition}\label{prop4}
If $g(n) = o(n^{2/3})$, for each sub-optimal $i \neq i^*$, the following holds almost surely:
\begin{equation}
\liminf_n \frac{ \Delta_i T^i_{\ismi}(n) - g(n)}{ \sqrt{ g(n) \ln \ln g(n) } } \geq  -\frac{ 3\sigma_i \sqrt{2}}{  \sqrt{ \Delta_i } }.
\end{equation}
\end{proposition}

Again, we leave as an open problem that of extending these results to the case of non-unique optimal bandits.

\section{Comparison between Policies}\label{sec:comparison}
We have established two policies, \gfp\ and \ism, that each achieve $O(g(n))$ pseudo-regret, almost surely. The question of which policy is ``better'' is not necessarily well posed. For one thing, the asymptotic pseudo-regret growth of either policy can be improved by picking a slower $g$. In this sense, there is certainly no ``optimal'' policy as there will always be a slower $g$. For a fixed $g$, however, the question of which policy is better  becomes context specific: for some bandit distributions, the order constant of the \gfp\ policy, $S_\Delta$, will be smaller than the order constant of the \ism\ policy, $K-1$; for some bandit distributions, the comparison will go the other way.

In terms of the results presented here, the pseudo-regret of the \gfp\ policy is much more tightly controlled, Proposition \ref{thm:gfg-strong} bounding the fluctuations in pseudo-regret around $S_\Delta g(n)$ by at most a constant - indeed, at most $S_\Delta$. The bounds on the \ism\ policy however are $O( \sqrt{ g(n) \ln \ln g(n) } )$. But, this additional control of the \gfp\ policy comes at a cost. It follows from the proof of Proposition \ref{thm:gfg-strong} that for sub-optimal $i$, for all large $n$,
\begin{equation}
T^i_{\gfpi}(n) \approx g(n).
\end{equation}
However, for the \ism\ policy, following the proof of \ref{thm:uindex-unique}, for all sub-optimal $i$, and large $n$,
\begin{equation}
T^i_{\ismi}(n) \approx \frac{g(n)}{ \Delta_i }.
\end{equation}
It is clear from this that the \gfp\ policy is in some sense the more democratic of the two, sampling all sub-optimal bandits equally, regardless of quality. The \ism\ policy is the more meritocratic, sampling   sub-optimal bandits more rarely the farther they are from the optimum. This has the effect of boosting the sampling of bandits near the optimum, but this effect is somewhat counterbalanced as they contribute less to the pseudo-regret.

\section{Relaxing Assumptions: i.i.d. Bandits}\label{sec:relax}
The assumption that the results from each bandit are i.i.d. is fairly standard - the problem is generally phrased as a matter of knowledge discovery about a set of unknown distributions, though the use of repeated measurements. However, it is interesting to observe that this assumption actually plays no part in the results and proofs present in this paper. The sole distributional property that mattered for establishing the policies as $g$-good was the assumption that for each bandit there existed some finite $\mu_i$ such that $\bar X^i_k \rightarrow \mu_i$ almost surely with $k$ (though the Law of Iterated Logarithms was utilized to great effect in bounding the remainder terms). In fact, the expected values of the individual $X^i_j$ need not be $\mu_i$, nor must the $X^i_k$ be independent of each other for a given $i$. Further, it is never necessary that the bandits themselves be independent of each other! In that regard, the results herein are actually quite general statements about minimizing pseudo-regret under arbitrary multidimensional stochastic processes that satisfy that strong large number law-type requirement.

However, a word of caution is due: removing the restrictions on $\{ X^i_k \}_{k \geq 1}$ in this way, while not influencing the proofs of the results presented here, does somewhat call into question the definition of ``pseudo-regret'' as given in Eq. \eqref{eqn:pseudo-regret}. The individual sample means freed, it is not necessarily reasonable to define a finite horizon pseudo-regret, $\tilde{R}_\pi(n)$, in terms of the infinite horizon means, $\{ \mu_i \}$. For instance, it is no longer necessarily true that the optimal, complete knowledge policy on any finite horizon is simply to activate a bandit with infinite horizon mean $\mu^*$ at every point. A more applicable definition of pseudo-regret would have to take into account what is reasonable to know or measure about the state of each bandit in finite time.


{\bf Acknowledgement:} 
We would like to acknowledge support for this project
from the National Science Foundation (NSF grant CMMI-14-50743).



\vskip 0.2in
 
\bibliography{mab2015}

\appendix

\section{Proof of Proposition \ref{thm:gfg-strong}}\label{apx:gfg-strong}


\begin{proof}
To prove Proposition \ref{thm:gfg-strong}, it will suffice to show the following: For all $   i: \mu_i \neq \mu^*$ and all  $\delta > 0$, \ase    a finite time $T_\delta < \infty$ such that that,
\begin{equation}\label{eqn:goal}
g(t) - 2\delta \leq  T^i_{\pi^F_g}(t) \leq \ceil{ g(t) } \ 
\mbox{   $\forall t \geq T_\delta$}.
\end{equation}
Theorem \ref{thm:gfg-strong} follows from this result and Eq. \eqref{eqn:pseudo-regret}, with the appropriate choice of $\delta$.

Without loss of generality, we may restrict ourselves to $\delta < 1/2$.

  As a preliminary step: Based on the properties of $g$, if $K$ is the total number of bandits, there exists a finite, not random, time $t_\delta$ such that , the following is true: 
\begin{equation}\label{in:gN}
g(t + K) < g(t) + \delta \, , \ \forall \ t \geq t_\delta.
\end{equation}
This follows from the observation that $g(t+K) \leq g(t) + g'(t)K$, and that $g'(t) \to 0$.  

When implementing a \gfp\ policy $\pi^F_g$ (hereafter referenced simply as $\pi$), there are essentially two alternating phases (or modes)  of the policy: ``catch up'' and ``play the winner''. During ``catch up'', some number of bandits have fewer than $g$ activations (the sub-$g$ bandits), and they are activated until all bandits have at least $g$ activations. During ``play the winner'', each bandit has at least $g$ activations, and the bandit with the current greatest sample mean is activated. These phases can be seen as governed by the function $\Delta(t) = g(t) - \min_i T^i_\pi(t)$ so that when $\Delta(t) >0,$ the policy is in ``catch up'' mode, when $\Delta(t) \le 0 ,$  the policy is in ``play the winner'' mode.

Having activated bandits according to policy $\pi$ up to time $t_\delta$, suppose that $\Delta(t_\delta) > 0$, hence the policy enters or is in a period of ``catch up''. Let $d ( = d(t_\delta))$ be the number of sub-$g$ bandits at time $t_\delta$. Because $g$ is increasing, and there are $d$ sub-$g$ bandits at time $t_\delta$, it will take at least $d$ ``catch up'' activations before the policy enters a period of ``play the winner'' ($\Delta \leq 0$). Consider activating bandits according to policy $\pi$ for $d$ activations. Note, $d \leq K$, so from \inqqref{in:gN} and increasing property of $g$ we have: $g(t_\delta + d) < g(t_\delta) + \delta$. Additionally, $\min_i T^i_\pi(t_\delta + d) \geq \min_i T^i_\pi(t_\delta) + 1$, as every bandit realizing the minimum activations will have been activated at least once. It follows that
\begin{equation}
\begin{split}
\Delta(t_\delta + d) & = g(t_\delta + d) - \min_i T^i_\pi(t_\delta + d)\\
& < g(t_\delta) + \delta - \min_i T^i_\pi(t_\delta) - 1\\
& = \Delta(t_\delta) - (1-\delta).
\end{split}
\end{equation}

Hence, after a period of $d$ activations from time $t_\delta$, the spread $\Delta$ has decreased by at least $1-\delta$. Repeating this argument, based on the number of sub-$g$ bandits (if any) at time $t_\delta+d$, it is clear that eventually - {\sl in finite time} - a time $T_\Delta < \infty$ is reached such that $\Delta(T_\Delta) \leq 0$. At this point, all bandits have been activated at least $g$ times, and the policy enters a period of ``play the winner''. We observe the loose, but sample-path-wise, bound that,
\begin{equation}
 T_\Delta \leq t_\delta + K\frac{(\Delta(t_\delta))^+}{1-\delta} \leq t_\delta + K\frac{g(t_\delta)}{1-\delta} < \infty,
 \end{equation}
since $\Delta(t) \leq g(t)$ always, and at every step the number of sub-$g$ bandits is at most $K$. Observe that if in fact $\Delta(t_\delta) \leq 0$, then we may take $T_\Delta = t_\delta$.

Having entered a period of $\Delta \leq 0$ or ``play the winner'' at time $T_\Delta$, let $t \geq T_\Delta$ such that $\Delta(t) \leq 0$ but $\Delta(t+1) > 0$. That is, in the transition from time $t$ to $t+1$, $g$ surpasses the number of activations of some bandits and the policy enters a period of ``catch up''. At such a point, we have the following relations:
\begin{equation}\label{eqn:gp1}
\min_i T^i_\pi (t + 1) < g(t+1)  < g(t) + \delta  \leq \min_i T^i_\pi(t) + \delta.
\end{equation}
The first inequality is simply that $\Delta(t+1) > 0$, the second following since $t \geq t_\delta$, and the last since $\Delta(t) \leq 0$. However, since the $T^i_\pi$ are integer valued and non-decreasing, the above yields
\begin{equation}\label{eqn:gp2}
\min_i T^i_\pi(t+1) = \min_i T^i_\pi(t).
\end{equation}
Combining Eqns. \eqref{eqn:gp1}, \eqref{eqn:gp2} yields the important relation that $\Delta(t+1) < \delta$. Note additionally,
\begin{equation}\label{eqn:gp3}
g(t+1) < g(t) + \delta  \leq \min_i T^i_\pi(t) + \delta < \min_i T^i_\pi(t+1) + 1.
\end{equation}
Again noting the $T^i_\pi$ are integer valued, this implies that while there are sub-$g$ bandits at time $t+1$, the only sub-$g$ bandits are those that realize the minimum number of activations $\min_i T^i_\pi(t+1)$. All other bandits have activations strictly greater than $g$. Let the number of sub-$g$ bandits at time $t+1$ again be denoted $d = d(t+1)$. For $d' < d\ ( \leq K)$ additional activations under $\pi$, in the ``catch up'' phase, we have that $\min_i T^i_\pi(t + 1 + d') = \min_i T^i_\pi(t + 1)$ and $g(t + 1 + d') < g(t+1) + \delta$. Hence, $\Delta(t+1+d') < \Delta(t + 1) + \delta < 2\delta$. For $d$ additional activations after time $t+1$, each sub-$g$ bandit has been activated once, raising the minimum number of activations by 1: $\min_i T^i_\pi(t + 1 + d) = \min_i T^i_\pi(t+1) + 1$. Additionally, $g(t + 1 + d) < g(t+1) + \delta$, hence $\Delta(t+1+d) < \Delta(t+1) - \delta < 0$.

We see therefore that after $T_\Delta$, at any point at which $\Delta$ becomes positive after being at most zero, it is at most $2\delta$ for a finite time - the ``catch up'' phase -  before becoming negative. Hence it follows, that for $t \geq T_\Delta$, $\Delta(t) \leq 2\delta$, or for each $i$
\begin{equation}\label{eqn:forcingfn}
g(t) - 2\delta \leq T^i_\pi(t).
\end{equation}

Note, this is true for all $i$. This acts as justification for the description of $g$ as the ``forcing function'', as the policy forces all activations to grow at least at $g$ asymptotically.

Since $g$ is unbounded and increasing, all populations are sampled infinitely often over time. Taking the strong law of large numbers to hold, for every $\epsilon > 0$ and each $i$, there exists almost surely some finite $N^i_\epsilon$ such that $\bar X^i_k \in [\mu_i - \epsilon, \mu_i + \epsilon]$ for all $k \geq N^i_\epsilon$. It is worth noting here that while such a $N^i_\epsilon$ exists, it is  random and unknowable to the investigator. Because of the properties of $g$, we may define a finite $T^i_\epsilon > T_\Delta$ such that $N^i_\epsilon \leq g(T^i_\epsilon) - 2\delta$. By Eq. \eqref{eqn:forcingfn}, we have that for all $t \geq T^i_\epsilon$,
\begin{equation}
\bar X^i_{T^i_{\pi}(t)} \in [\mu_i - \epsilon, \mu_i + \epsilon].
\end{equation}

Hence we have for each population, for every $\epsilon > 0$, there exists almost surely a finite random time $T_\epsilon = \max_i T^i_\epsilon < \infty$ past which the sample mean is trapped within the $\mu_i \pm \epsilon$ interval.

Fix $\epsilon$ sufficiently small, so as to distinguish $\mu^*$ from the other means (i.e., $[\mu^* - \epsilon, \mu^* + \epsilon] \cap [\mu_i - \epsilon, \mu_i + \epsilon] = \varnothing$ for all $i: \mu_i \neq \mu^*$). By the previous observations, we have therefore that for all $t \geq T_\epsilon$, for all sub-optimal $i$ and any optimal $i^*$,
\begin{equation}
\bar X^{i^*}_{T^{i^*}_{\pi}(t)} > \bar X^i_{T^i_{\pi}(t)}.
\end{equation}

In short, almost surely there exists a finite time $T_\epsilon$ past which the sample means of sub-optimal bandits are always inferior to the sample mean of any optimal bandit.

By the structure of the policy $\pi$, for all $t \geq T_\epsilon$, sub-optimal populations are only activated during the $g$-forced ``catch up'' periods. If at time $T_\epsilon$, the number of times a sub-optimal bandit $i$ has been activated is greater than $g$ - for instance due to it, at some point, having the largest sample mean during a ``play the winner'' period - that population will not be sampled again until $g$ has increased to overcome this temporary excess. As $g$ is increasing and unbounded, this must occur in finite time. Once this occurs, as observed previously, $g$ can only exceed $T^i_\pi$ by at most $2\delta$ before bandit $i$ is again activated, raising $T^i_\pi$ above $g$ once more. As this ``catch up'' is the only time bandit $i$ is activated, and $\delta < 1/2$, it follows that there exists some finite time $\tilde T^i_\epsilon > T_\epsilon$ such that for $t \geq \tilde T^i_\epsilon$, $T^i_\pi(t) \leq \ceil{ g(t) }$. Taking $T_\delta = \max_{i : \mu_i \neq \mu^*} \tilde T^i_\epsilon$, and noting that $t_\delta \leq T_\Delta \leq T_\epsilon \leq T_\delta < \infty$, we have that for $t \geq T_\delta$, for all sub-optimal $i$,
\begin{equation}
g(t) - 2\delta \leq T^i_\pi(t) \leq \ceil{ g(t) }.
\end{equation}
\end{proof}

\section{Proofs of Propositions \ref{prop1}, \ref{prop2}  }\label{apx:cub}

In this section, $\pi$ refers to a \ism\ policy as in   \eqqref{eqn:pi-g-2}. The results to follow depend on the following lemma.

\begin{lemma}\label{lem1}
Under the assumption of Eq. \eqref{eqn:stln}, for each $i$, and  for any  $\epsilon > 0$, the inequality: $$u_i(j,k) < \mu_i - \epsilon$$ holds for only finitely many $(j,k)$-pairs, almost surely.
\end{lemma}

\begin{proof}
As an application of the strong law, almost surely there is some finite $N^i_\epsilon$ such that $\bar X^i_k > \mu - \epsilon/2 ,$ for all $k \geq N^i_\epsilon$. For such $k$, as $g$ is positive, $u_i(j,k)=
\bar X^i_k + g(j)/k  \geq \mu_i - \epsilon ,$ \  for all $j$. For any $k < N^i_\epsilon$, the relation $u_i(j,k) = \bar X^i_k + g(j)/k   < \mu_i - \epsilon$ may be true only for finitely many $j$ since $g$ is increasing.
\end{proof}

\begin{proof}{\bf  of Proposition \ref{prop1}.}
For $i \neq i^*$, we define the following quantities. Taking $\epsilon> 0$, and $2\epsilon < \mu^* - \mu_i$, and  $n \geq K$,
\begin{equation}
\begin{split}
n^i_1(n, \epsilon) & =  \sum_{t = N}^n \mathbf{1}\{ \pi(t + 1) = i, u_i(t, T^i_\pi(t)) \geq \mu^* - \epsilon,\bar  X^i_{ T^i_\pi(t) } \leq \mu^i + \epsilon \}\\
n^i_2(n, \epsilon) & = \sum_{t = N}^n \mathbf{1}\{ \pi(t + 1) = i, u_i(t, T^i_\pi(t)) \geq \mu^* - \epsilon,\bar  X^i_{ T^i_\pi(t) } > \mu^i + \epsilon \}\\
n^i_3(n, \epsilon) & = \sum_{t = N}^n \mathbf{1}\{ \pi(t + 1) = i, u_i(t, T^i_\pi(t)) < \mu^* - \epsilon \}.
\end{split}
\end{equation}
Hence we have the following relationship,
\begin{equation}
\label{eqn:t-parts}
T^i_{\pi}(n + 1) = 1 + \sum_{t = N}^n \mathbf{1}\{ \pi(t + 1) = i \} = 1 + n^i_1(n, \epsilon) + n^i_2(n, \epsilon) + n^i_3(n, \epsilon).
\end{equation}

The proof proceeds via a pointwise bound on each of the three terms. For the first term,
\begin{equation}
\begin{split}
n^i_1(n, \epsilon) & \leq \sum_{t = N}^n \mathbf{1}\{ \pi(t + 1) = i, \mu^i + \epsilon + g(t)/T^i_{\pi}(t) \geq \mu^* - \epsilon \} \\
& = \sum_{t = N}^n \mathbf{1}\{ \pi(t + 1) = i, g(t)/((\mu^*-\mu_i) - 2\epsilon) \geq T^i_{\pi}(t) \} \\
& \leq \sum_{t = N}^n \mathbf{1}\{ \pi(t + 1) = i, g(n)/((\mu^*-\mu_i) - 2\epsilon ) \geq T^i_{\pi}(t) \} \\
& \leq \frac{g(n)}{(\mu^*-\mu_i) - 2\epsilon} + 1.
\end{split}
\end{equation}
The last inequality comes from viewing $T^i_{\pi}(t)$ as a sum of $\mathbf{1}\{ \pi(t + 1) = i \}$ indicators, and seeing that the condition on it bounds the number of non-zero terms in this sum.

For the second term,
\begin{equation}
\begin{split}
n^i_2(n, \epsilon) & \leq \sum_{t = N}^n \mathbf{1}\{ \pi(t + 1) = i, \bar  X^i_{ T^i_{\pi}(t) } > \mu^i + \epsilon \}\\
& = \sum_{t = N}^n \sum_{k = 1}^t \mathbf{1}\{ \pi(t + 1) = i, \bar  X^i_k > \mu^i + \epsilon, T^i_{\pi}(t) = k \} \\
& = \sum_{t = N}^n \sum_{k = 1}^t \mathbf{1}\{ \pi(t + 1) = i, T^i_{\pi}(t) = k \} \mathbf{1}\{\bar  X^i_k > \mu^i + \epsilon\}\\
& \leq \sum_{k = 1}^n \mathbf{1}\{\bar  X^i_k > \mu^i + \epsilon\} \sum_{t = k}^n \mathbf{1}\{ \pi(t + 1) = i, T^i_{\pi}(t) = k \} \\
& \leq \sum_{k = 1}^n \mathbf{1}\{\bar  X^i_k > \mu^i + \epsilon\}.
\end{split}
\end{equation}
The last inequality holds as, for a given $k$, $\{\pi(t + 1) = i, T^i_{\pi}(t) = k\}$ may be true for only one $t$. Taking it one step further, we have
\begin{equation}
n^i_2(n, \epsilon) \leq \sum_{k = 1}^\infty \mathbf{1}\{\bar  X^i_k > \mu^i + \epsilon\},
\end{equation}
and since the strong law of large numbers is taken to hold, we have therefore that $n^i_2(n)$ is almost surely bound by a finite constant, for all $n \geq K$.

For the third term, note that from the structure of the policy, a population is only sampled if it has the maximal current index. Hence, if $\pi(t + 1) = i$, it must be true that $u_{i^*}(t, T^{i^*}_{\pi}(t)) \leq u_{i}(t, T^{i}_{\pi}(t))$. Hence we have the bound,
\begin{equation}
\begin{split}
n^i_3(n, \epsilon) & \leq \sum_{t = N}^n \mathbf{1}\{ \pi(t + 1) = i, u_{i^*}(t, T^{i^*}_{\pi}(t)) < \mu^* - \epsilon \}\\
& \leq \sum_{t = N}^n \mathbf{1}\{ u_{i^*}(t, T^{i^*}_{\pi}(t)) < \mu^* - \epsilon \}\\
& \leq \sum_{t = N}^\infty \mathbf{1}\{ u_{i^*}(t, T^{i^*}_{\pi}(t)) < \mu^* - \epsilon \}.
\end{split}
\end{equation}
From the prior observation about the form of the index, Lemma \ref{lem1}, we have that $u_{i^*}(t, T^{i^*}_{\pi}(t)) < \mu^* - \epsilon$ is true for only finitely many $t$, almost surely. Hence, from the above bound, $n^i_3(n)$ is almost surely bound by a finite constant, for all $n \geq K$.

Combining the above results bounding $n^i_1, n^i_2, n^i_3$ with Eq.\eqref{eqn:t-parts}, and observing too that $T^i_{\pi}(n) \leq T^i_{\pi}(n+1)$, we have that almost surely there exists some finite $C^i_\epsilon$ such that for all $n \geq K$,
\begin{equation}
T^i_{\pi}(n) \leq \frac{g(n)}{(\mu^*-\mu_i) - 2\epsilon} + C^i_\epsilon.
\end{equation}
\end{proof}
\linebreak[3]

\begin{proof}{\bf  of Proposition \ref{prop2}.}
Define a constant $P_\Delta = \sum_{i \neq i^*} 1/(\mu^* - \mu_i)$. Taking $\epsilon < \min_{j \neq i^*} (\mu^* - \mu_j)/2$, we may apply Prop. \ref{prop1} to yield   for each $i \neq i^*$, \ase a finite $N^i_\epsilon$ such that $T^i_{\pi}(n) \leq (1 + \epsilon) g(n)/(\mu^* - \mu_i)$ for all $n \geq N^i_\epsilon$. Taking $N_\epsilon = \max_{i \neq i^*} N^i_\epsilon$, summing over these relations and taking $n \geq N_\epsilon$, 
\begin{equation}\label{eqn:activation-bound}
\sum_{i \neq i^*} T^i_{\pi}(n) \leq (1+\epsilon)g(n) P_\Delta.
\end{equation}

The sum above equals the number of activations of sub-optimal bandits up to and including time $n$. As the total number of bandit activations up to time $n$ is $n$, we have from the above that $T^{i^*}_{\pi}(n) \geq n - O(g(n))$. 

Trivially from this, the optimal bandit $i^*$ is activated infinitely often, approaching full density of activations as $n$ increases.

Given this linear lower bound on $T^{i^*}_{\pi}$, it follows that $u_{i^*}( n, T^{i^*}_{\pi}(n) )$ converges to $\mu^*$, almost surely. Hence, almost surely there exists a finite $\tilde N_\epsilon$ such that for $n \geq \tilde N_\epsilon$, $u_{i^*}( n, T^{i^*}_{\pi}(n) ) \leq \mu^* + \epsilon$. As under this policy a bandit is only activated when it has the maximal index, it follows that infinitely often (on the activations of $i^*$), the indices of all sub-optimal bandits are at most $\mu^* + \epsilon$. Given the structure of the indices, it follows that these sub-optimal bandits must be activated infinitely often as well. Hence, almost surely, $T^i_{\pi}(n)$ increases without bound, for all $i$. Applying the strong law here, since there are finitely many bandits being considered, \ase a finite  ``$\epsilon$-trapping time'', $\tilde N^{\text{trap}}_\epsilon, $ such that 
  $${\bar X}^i_{T^i_{\pi}(n)} \in \left[ \mu_i - \epsilon, \mu_i + \epsilon \right], \     \forall n \geq \tilde N^\text{trap}_\epsilon \ and \  \forall i.$$

Let $\{n_k\}_{k \geq 0}$ be the infinite sequence of times at which bandit $i^*$ has the current optimal index (and hence is activated next). For a given $i \neq i^*$, we have that for all sufficiently large $k$ $(n_k \geq \tilde N^\text{trap}_\epsilon)$,
\begin{equation}\label{eqn:first-max}
\begin{split}
\max_{n_k \leq n \leq n_{k+1}} u_i(n, T^i_{\pi}(n)) & \leq (\mu_i + \epsilon) + \frac{g(n_{k+1})}{T^i_{\pi}(n_k)}\\
& = (\mu_i + \epsilon) + \frac{g(n_{k+1})}{g(n_k)} \frac{ g(n_k) }{ T^i_{\pi}(n_k) }\\
& = (\mu_i + \epsilon) + \frac{g(n_{k+1})}{g(n_k)} \left(  u_i(n_k, T^i_{\pi}(n_k)) - \bar X^i_{T^i_{\pi}(n_k)} \right)\\
& \leq (\mu_i + \epsilon) + \frac{g(n_{k+1})}{g(n_k)} \left( u_i(n_k, T^i_{\pi}(n_k)) - (\mu_i - \epsilon) \right).\\
\end{split}
\end{equation}

Additionally, however, at time $n_k$ bandit $i^*$ has the largest index. For sufficiently large $k$ $(n_k \geq \tilde N_\epsilon)$, this index must be at most $\mu^* + \epsilon$. Hence for $n_k > \max( \tilde N_\epsilon, \tilde N^\text{trap}_\epsilon)$, for $i \neq i^*$ we have that $u_{i}( n_k, T^{i}_{\pi}(n_k) ) \leq u_{i^*}( n_k, T^{i^*}_{\pi}(n_k) ) \leq \mu^* + \epsilon$, and
\begin{equation}
\begin{split}
\max_{n_k \leq n \leq n_{k+1}} u_i(n, T^i_{\pi}(n)) & \leq (\mu_i + \epsilon) + \frac{g(n_{k+1})}{g(n_k)} \left( (\mu^* + \epsilon) - (\mu_i - \epsilon) \right)\\
& = (\mu_i + \epsilon) + \frac{g(n_{k+1})}{g(n_k)} \left( \mu^* - \mu_i + 2\epsilon \right).
\end{split}
\end{equation}


Since we took $g$ to be concave, $g(n_{k+1}) \leq g(n_k) + (n_{k+1} - n_k)g'(n_k)$. The difference $n_{k+1} - n_k - 1$ is the number of sub-optimal bandit activations between the $k$ and $k+1$-th activations of bandit $i^*$. This is bound from above by the {\sl total} number of sub-optimal activations prior to time $n_{k+1}$, which by Eq. \eqref{eqn:activation-bound} is at most $(1+\epsilon)g(n_{k+1}) P_\Delta$ for all $n_{k+1} \geq N_\epsilon$. Hence,
\begin{equation}
g(n_{k+1}) \leq g(n_k) + ((1+\epsilon)g(n_{k+1}) P_\Delta + 1)g'(n_k).
\end{equation}
As $g' \rightarrow 0$, for all sufficiently large $k$, we have that $(1+\epsilon)P_\Delta g'(n_k) < 1$ and
\begin{equation}\label{eqn:g-ratio}
\frac{ g(n_{k+1}) }{g(n_k)} \leq \frac{1 + \frac{g'(n_k)}{g(n_k)}}{1 - (1+\epsilon)P_\Delta g'(n_k)}.
\end{equation}

As $g$ is taken to be increasing, and $g'$ is taken to limit to $0$, we have from the above that there is some finite $\tilde N^g_\epsilon$ such that for all sufficiently large $k$ $(n_k \geq N^g_\epsilon)$, $g(n_{k+1})/g(n_k) \leq 1 + \epsilon$. Hence, for $n_k \geq \max( N_\epsilon, \tilde N_\epsilon, \tilde N^\text{trap}_\epsilon, \tilde N^g_\epsilon)$,
\begin{equation}\label{eqn:third-max}
\max_{n_k \leq n \leq n_{k+1}} u_i(n, T^i_{\pi}(n)) \leq (\mu_i + \epsilon) + (1 + \epsilon)(\mu^*-\mu_i + 2\epsilon).
\end{equation}

Let $N^K_\epsilon = \min \{ n_k : n_k > \max( N_\epsilon, \tilde N_\epsilon, \tilde N^\text{trap}_\epsilon, \tilde N^g_\epsilon) \}  < \infty$. As the upper bound above no longer depends on $k$, we have that for $n \geq N^K_\epsilon$,

\begin{equation}\label{eqn:fourth-max}
u_i(n, T^i_{\pi}(n)) \leq (\mu_i + \epsilon) + (1 + \epsilon)(\mu^*-\mu_i + 2\epsilon).
\end{equation}

Observing that ${\bar X}^i_{T^i_{\pi}(n)} \geq \mu_i - \epsilon$, the above yields $\mu_i - \epsilon + g(n)/T^i_{\pi}(n)  \leq (\mu_i + \epsilon) + (1 + \epsilon)(\mu^*-\mu_i + 2\epsilon)$, or
\begin{equation}\label{eqn:fifth-max}
\frac{g(n)}{(1 + \epsilon)(\mu^*-\mu_i + 2\epsilon) + 2\epsilon} \leq T^i_{\pi}(n).
\end{equation}
\end{proof}

\section{Proofs of Propositions \ref{prop3}, \ref{prop4}  }\label{apx:lil}
We present the following preliminary bounds to aid in the proofs of Props. \ref{prop3}, \ref{prop4}. In this section, $\pi$ is taken to be an \ism\ policy as in Eq. \ref{eqn:pi-g-2}. Additionally, it is convenient to define 
\begin{equation}
P_\Delta = \sum_{i \neq i^*} \frac{1}{ \mu^* - \mu_i }.
\end{equation}

It follows from Props. \ref{prop1}, \ref{prop2} that for any $\epsilon > 0$, \ase  some finite $N_{\epsilon}$ such that for $n \geq N_\epsilon$, the following holds: for $i \neq i^*$,
\begin{equation}\label{eqn:g-bound}
\frac{1-\epsilon}{\mu^* - \mu_i} g(n) \leq T^i_{\pi}(n) \leq \frac{1+\epsilon}{\mu^* - \mu_i} g(n).
\end{equation}
And similarly, for the optimal bandit,
\begin{equation}
n - (1 + \epsilon) P_\Delta g(n) \leq T^{i^*}_{\pi}(n) \leq n - (1 - \epsilon) P_\Delta g(n).
\end{equation}
To simplify the case for the optimal bandit, slightly, it also holds that for all sufficiently large n, $T^{i^*}_{\pi}(n) \geq n/2$. We'll also observe here, as an aside, that for some finite $\tilde N_\epsilon$, $$(1-\epsilon)/(\mu^* - \mu_i) g(n) > 6,\ \ \mbox{for all $n \geq \tilde N_\epsilon, \ and \  i \neq i^*$} .$$

As each bandit is activated infinitely often, $T^i_\pi$ increases without bound with $n$, and hence we may apply the Law of the Iterated Logarithm in the following way: There exists a finite time $N'_\epsilon$ such that for $n \geq N'_\epsilon$, for each bandit $i$,
\begin{equation}
\lvert \bar X^i_{T^i_\pi(n)} - \mu_i \rvert \leq \sigma_i \sqrt{2}(1+\epsilon) \sqrt{ \frac{ \ln \ln T^i_\pi(n) }{ T^i_\pi(n) } }.
\end{equation}

However, since $\sqrt{ \ln \ln x / x }$ is decreasing for all $x \geq 6$, we may apply the above bounds to have that, for $n \geq \max( N_\epsilon, N'_\epsilon, \tilde N_\epsilon, 12 )$, for $i \neq i^*$,
\begin{equation}\label{eqn:lil-bound}
\lvert \bar X^i_{T^i_\pi(n)} - \mu_i \rvert \leq \sigma_i \sqrt{2}(1+\epsilon) \sqrt{ \frac{ \ln \ln \left(\frac{1-\epsilon}{\mu^* - \mu_i} g(n)\right) }{ \frac{1-\epsilon}{\mu^* - \mu_i} g(n) } },
\end{equation}
and for the optimal bandit,
\begin{equation}\label{eqn:lil-bound-optimal}
\lvert \bar X^{i^*}_{T^{i^*}_\pi(n)} - \mu^* \rvert \leq \sigma_{i^*} \sqrt{2}(1+\epsilon) \sqrt{ \frac{ \ln \ln (n/2) }{ n/2 } }.
\end{equation}

\begin{proof}{\bf  of Proposition \ref{prop3}.}
Let $1 > \epsilon > 0$. For $i \neq i^*$, let
\begin{equation}
h_i(t) =  \sigma_i \sqrt{2(\mu^* - \mu_i)}  \frac{ (1 + \epsilon)^2 }{ \sqrt{ 1 - \epsilon } } \sqrt{ \frac{ \ln \ln g(t) }{g(t)} }.
\end{equation}
Observe that $h_i \rightarrow 0$ from above as $t \rightarrow \infty$. Note that there exists a $T_\epsilon < \infty$ such that for $t \geq T_\epsilon$, $g(t)/( \mu^* - \mu_i - 2 h_i(t))$ is increasing. The proof proceeds analogously to the proof of Prop. \ref{prop1}, utilizing the improved iterated logarithm bounds above.

For $n \geq T_\epsilon$, define the following functions:
\begin{equation}
\begin{split}
\tilde n^i_1(n) & = \sum_{t = T_\epsilon}^n \mathbf{1}\{ \pi(t + 1) = i, u_i(t, T^i_\pi(t)) \geq \mu^* - h_i(t),\bar  X^i_{ T^i_\pi(t) } \leq \mu_i + h_i(t) \}\\
\tilde n^i_2(n) & = \sum_{t = T_\epsilon}^n \mathbf{1}\{ \pi(t + 1) = i, u_i(t, T^i_\pi(t)) \geq \mu^* - h_i(t),\bar  X^i_{ T^i_\pi(t) } > \mu_i + h_i(t) \}\\
\tilde n^i_3(n) & = \sum_{t = T_\epsilon}^n \mathbf{1}\{ \pi(t + 1) = i, u_i(t, T^i_\pi(t)) < \mu^* - h_i(t) \}.
\end{split}
\end{equation}

Hence, we have the following relationship, that for $n \geq T_\epsilon$,
\begin{equation}
T^i_\pi(n) \leq T_\epsilon + 1 + \tilde n^i_1(n) + \tilde n^i_2(n) + \tilde n^i_3(n).
\end{equation}

The proof proceeds as in the proof of Prop. \ref{prop1}, bounding each of the three terms. For the first,
\begin{equation}
\begin{split}
\tilde n^i_1(n) & \leq \sum_{t = T_\epsilon}^n \mathbf{1}\{ \pi(t + 1) = i, \mu_i + h_i(t) + g(t)/T^i_{\pi}(t) \geq \mu^* - h_i(t) \} \\
& = \sum_{t = T_\epsilon}^n \mathbf{1}\{ \pi(t + 1) = i, g(t)/((\mu^*-\mu_i) - 2 h_i(t)) \geq T^i_{\pi}(t) \} \\
& \leq \sum_{t = T_\epsilon}^n \mathbf{1}\{ \pi(t + 1) = i, g(n)/((\mu^*-\mu_i) - 2 h_i(n) ) \geq T^i_{\pi}(t) \} \\
& \leq \frac{g(n)}{(\mu^*-\mu_i) - 2 h_i(n)} + 1.
\end{split}
\end{equation}
As before, the last inequality comes from viewing $T^i_{\pi}(t)$ as a sum of $\mathbf{1}\{ \pi(t + 1) = i \}$ indicators, and seeing that the condition on it bounds the number of non-zero terms in this sum. It is also important to observe here that we are explicitly in a regime in which $g(t)/((\mu^*-\mu_i) - 2 h_i(t))$ is an increasing function with $t$.

For the second term,
\begin{equation}
\begin{split}
\tilde n^i_2(n) & \leq \sum_{t = T_\epsilon}^n \mathbf{1}\{ \pi(t + 1) = i, \bar  X^i_{ T^i_{\pi}(t) } > \mu_i + h_i(t) \}\\
& \leq \sum_{t = T_\epsilon}^n \mathbf{1}\{ \bar  X^i_{ T^i_{\pi}(t) } - \mu_i > h_i(t) \}\\
& \leq \sum_{t = T_\epsilon}^n \mathbf{1}\left\{ \sigma_i \sqrt{2}(1+\epsilon) \sqrt{ \frac{ \ln \ln \left(\frac{1-\epsilon}{\mu^* - \mu_i} g(t)\right) }{ \frac{1-\epsilon}{\mu^* - \mu_i} g(t) } } > h_i(t) \right\}\\
\end{split}
\end{equation}
The last inequality holds, by the iterated logarithm bound in Eq. \eqref{eqn:lil-bound}. Taking it one step further, we have
\begin{equation}
\tilde n^i_2(n) \leq \sum_{t = T_\epsilon}^\infty \mathbf{1}\left\{ \frac{ \sigma_i \sqrt{2}(1+\epsilon)}{ h_i(t) } \sqrt{ \frac{ \ln \ln \left(\frac{1-\epsilon}{\mu^* - \mu_i} g(t)\right) }{ \frac{1-\epsilon}{\mu^* - \mu_i} g(t) } } > 1 \right\}.
\end{equation}
Note that as
\begin{equation}\label{eqn:ci-limit}
\lim_t \frac{ \sigma_i \sqrt{2}(1+\epsilon)}{ h_i(t) } \sqrt{ \frac{ \ln \ln \left(\frac{1-\epsilon}{\mu^* - \mu_i} g(t)\right) }{ \frac{1-\epsilon}{\mu^* - \mu_i} g(t) } } = \frac{ 1 }{ 1 + \epsilon } < 1,
\end{equation}
the event indicated in the above sum bounding $\tilde n^i_2(n)$ may occur only finitely may times, almost surely. Hence, $\tilde n^i_2(n)$ is almost surely bound by a finite constant, for all $n \geq T_\epsilon$.

For the third term, as before, by the structure of the policy, a population is only sampled if it has the maximal current index. Hence, if $\pi(t + 1) = i$, it must be true that $u_{i^*}(t, T^{i^*}_{\pi}(t)) \leq u_{i}(t, T^{i}_{\pi}(t))$. It follows that
\begin{equation}
\begin{split}
\tilde n^i_3(n) & \leq \sum_{t = T_\epsilon}^n \mathbf{1}\{ \pi(t + 1) = i, u_{i^*}(t, T^{i^*}_{\pi}(t)) < \mu^* - h_i(t)\}\\
& \leq \sum_{t = T_\epsilon}^n \mathbf{1}\{ u_{i^*}(t, T^{i^*}_{\pi}(t)) < \mu^* - h_i(t) \}\\
& = \sum_{t = T_\epsilon}^n \mathbf{1}\left\{ \bar X^{i^*}_{T^{i^*}_\pi(t)} + \frac{g(t)}{ T^{i^*}_{\pi}(t) } < \mu^* - h_i(t) \right\}\\
& \leq \sum_{t = T_\epsilon}^n \mathbf{1}\left\{ -\sigma_{i^*} \sqrt{2}(1+\epsilon) \sqrt{ \frac{ \ln \ln (t/2) }{ t/2 } }+ \frac{g(t)}{ T^{i^*}_{\pi}(t) } < -h_i(t) \right\},
\end{split}
\end{equation}
the last equation coming from the iterated logarithm bound for the optimal bandit, Eq. \eqref{eqn:lil-bound-optimal}. As a final simplification,
\begin{equation}
\tilde n^i_3(n) \leq \sum_{t = T_\epsilon}^\infty \mathbf{1}\left\{ -\sigma_{i^*} \sqrt{2}(1+\epsilon) \sqrt{ \frac{ \ln \ln (t/2) }{ t/2 } } < -h_i(t) \right\}.
\end{equation}

If $g(n) = o( n / \ln \ln n )$, it is easy to verify that the indicated event in the above sum can only occur for finitely many $t$. Hence, by the above, there is a finite constant bounding $\tilde n^i_3(n)$ for all $n \geq T_\epsilon$.

Combining the above results, there is a finite constant $D^\epsilon_i$ such that for all $n \geq T_\epsilon$,
\begin{equation}
T^i_\pi(n) \leq \frac{g(n)}{(\mu^*-\mu_i) - 2 h_i(n)} + D^\epsilon_i.
\end{equation}
We have from this that
\begin{equation}
(\mu^* - \mu_i) T^i_\pi(n) - g(n) \leq g(n) \frac{2 h_i(n)}{(\mu^*-\mu_i) - 2 h_i(n)}  + (\mu^* - \mu_i) D^\epsilon_i.
\end{equation}
For a fixed $\epsilon > 0$, the above yields (taking the limit, given the choice of $h_i(n)$),
\begin{equation}
\limsup_n \frac{ (\mu^* - \mu_i) T^i_\pi(n) - g(n) }{ \sqrt{ g(n) \ln \ln g(n) } } \leq \frac{2 \sigma_i \sqrt{2}  (1 + \epsilon)^2 }{\sqrt{\mu^* - \mu_i}\sqrt{ 1- \epsilon}}.
\end{equation}

As the above holds for all $\epsilon > 0$, this yields, almost surely,
\begin{equation}
\limsup_n \frac{ (\mu^* - \mu_i) T^i_\pi(n) - g(n) }{ \sqrt{ g(n) \ln \ln g(n) } } \leq \frac{2 \sigma_i \sqrt{2} }{\sqrt{\mu^* - \mu_i}}.
\end{equation}
\end{proof}

\begin{proof}{\bf  of Proposition \ref{prop4}.}
Let $ \epsilon \in ( 0,1) . $ Recall from the proof of Prop. \ref{prop2} the infinite sequence $\{ n_k \}_{k \geq 0}$ of times at which the index of the optimal bandit $i^*$ is maximal. For notational convenience, we will write $u_i(n) = u_i(n, T^i_\pi(n))$, and for $i \neq i^*$, we define
\begin{equation}
U^i_k = \max_{n_k \leq n \leq n_{k+1}} u_i(n),
\end{equation}
and
\begin{equation}
M^i_k = \max_{n_k \leq n \leq n_{k+1}} \bar X^i_{T^i_\pi(n)}.
\end{equation}

We have the following relations,
\begin{equation}
\begin{split}
U^i_k & \leq \left( \max_{n_k \leq n' \leq n_{k+1}} \bar X^i_{T^i_\pi(n')} \right) + \frac{g(n_{k+1})}{T^i_{\pi}(n_k)}\\
& = M^i_k + \frac{g(n_{k+1})}{g(n_k)}\frac{g(n_{k})}{T^i_{\pi}(n_k)}\\
& = M^i_k + \frac{g(n_{k+1})}{g(n_k)}\left( u_i(n_k) - \bar X^i_{T^i_\pi(n_k)} \right)\\
& \leq M^i_k + \frac{g(n_{k+1})}{g(n_k)}\left( u_{i^*}(n_k) - \bar X^i_{T^i_\pi(n_k)} \right). 
\end{split}
\end{equation}

For $n$ such that $n_k \leq n \leq n_{k+1}$, trivially $u_i(n) \leq U^i_k$. It follows that
\begin{equation}
\frac{g(n)}{T^i_\pi(n)} \leq \left( M^i_k -  \bar X^i_{T^i_\pi(n)} \right) +  \frac{g(n_{k+1})}{g(n_k)}\left( u_{i^*}(n_k) - \bar X^i_{T^i_\pi(n_k)} \right).
\end{equation}

Defining the following terms for space,
\begin{equation}
\begin{split}
A_{n, k} & = \left( M^i_k -  \bar X^i_{T^i_\pi(n)} \right), \\
B_{k} & = \frac{ g(n_{k+1}) }{ g(n_k) } u_{i^*}(n_k) - \mu^*,\\
C_{k} & = \frac{ g(n_{k+1}) }{ g(n_k) } \bar X^i_{T^i_\pi(n_k)} - \mu_i,\\
\Delta(n) & = g(n) - (\mu^* - \mu_i) T^i_\pi(n),
\end{split}
\end{equation}
The above relation may be rearranged to yield
\begin{equation}
\Delta(n)/T^i_\pi(n) \leq A_{n, k} + B_{k}  - C_{k}.
\end{equation}

We may apply the iterated logarithm bounds of Eq. \eqref{eqn:lil-bound}, to yield a finite $K_A$ such that for $k \geq K_A$,
\begin{equation}
A_{n, k} \leq 2 \sigma_i \sqrt{2}(1+\epsilon) \sqrt{ \frac{ \ln \ln \left(\frac{1-\epsilon}{\mu^* - \mu_i} g(n_k)\right) }{ \frac{1-\epsilon}{\mu^* - \mu_i} g(n_k ) } }.
\end{equation}

Similarly, there is a finite $K_B$ such that for $k \geq K_B$, observing that for sufficiently large $k$, $T^{i^*}_{\pi}(n_k) \geq n_k/2$,
\begin{equation}
B_{k} \leq \frac{ g(n_{k+1}) }{ g(n_k) } \left( \mu^* + \sigma_{i^*} \sqrt{2}(1+\epsilon) \sqrt{ \frac{ \ln \ln (n_k/2) }{ n_k/2 } } + \frac{g(n_k)}{n_k/2} \right)  - \mu^*.
\end{equation}

And finally, there is a finite $K_C$ such that for $k \geq K_C$,
\begin{equation}
C_{k} \geq \frac{ g(n_{k+1}) }{ g(n_k) } \left( \mu_i - \sigma_i \sqrt{2}(1+\epsilon) \sqrt{ \frac{ \ln \ln \left(\frac{1-\epsilon}{\mu^* - \mu_i} g(n_k)\right) }{ \frac{1-\epsilon}{\mu^* - \mu_i} g(n_k) } } \right) - \mu_i.
\end{equation}

Rearranging terms for space again, for $k \geq \max( K_A, K_B, K_C )$ we have
\begin{equation}\label{eqn:tilde}
\Delta(n)/T^i_\pi(n) \leq A_{n, k} + B_{k}  - C_{k} \leq \tilde A_{k} + \tilde B_{k} + \tilde C_{k} + \tilde D_{k},
\end{equation}
where
\begin{equation}
\begin{split}
\tilde A_{k} & = (\mu^* - \mu_i)\left( \frac{g(n_{k+1})}{ g(n_k) } - 1\right) \\
\tilde B_{k} & = \sigma_i \sqrt{2}(1+\epsilon)\left(2   + \frac{ g(n_{k+1}) }{ g(n_k) }  \right)\sqrt{ \frac{ \ln \ln \left(\frac{1-\epsilon}{\mu^* - \mu_i} g(n_k)\right) }{ \frac{1-\epsilon}{\mu^* - \mu_i} g(n_k ) } } \\
\tilde C_{k} & = \sigma_{i^*} \sqrt{2}(1+\epsilon) \frac{ g(n_{k+1}) }{ g(n_k) }  \sqrt{ \frac{ \ln \ln (n_k/2) }{ n_k/2 } } \\
\tilde D_{k} & = \frac{ g(n_{k+1}) }{ g(n_k) } \frac{g(n_k)}{n_k/2}.  \\
\end{split}
\end{equation}
Noting that each of the above are positive, we have from Eq. \eqref{eqn:tilde},
\begin{equation}
\frac{ \Delta(n) }{ \sqrt{g(n) \ln \ln g(n)} } \leq \frac{ (\tilde A_{k} + \tilde B_{k} + \tilde C_{k} + \tilde D_{k}) T^i_\pi(n) }{ \sqrt{ g(n) \ln \ln g(n) } }.
\end{equation}
Note that, applying Eq. \eqref{eqn:g-bound} in this case, we have some finite $K_\epsilon$ such that for $k \geq K_\epsilon$,
\begin{equation}
T^i_\pi(n) \leq T^i_\pi(n_{k+1}) \leq \frac{1+\epsilon}{\mu^* - \mu_i} g(n_{k+1}).
\end{equation}
Recall from the proof of Prop. \ref{prop2} that there is a finite $K'_\epsilon$ such that for $k \geq K'_\epsilon$, $g(n_{k+1}) \leq (1 + \epsilon)g(n_k)$. Noting too that $g(n_k) \leq g(n)$, we have that for $k \geq \max( K_\epsilon, K'_\epsilon )$,
\begin{equation}\label{eqn:79}
\frac{ \Delta(n) }{ \sqrt{g(n) \ln \ln g(n)} } \leq \frac{ (\tilde A_{k} + \tilde B_{k} + \tilde C_{k} + \tilde D_{k})}{ \sqrt{ g(n_k) \ln \ln g(n_k) } } \frac{(1+\epsilon)^2}{ \left( \mu^* - \mu_i \right)} g(n_{k}).
\end{equation}
We have
\begin{equation}
\begin{split}
\frac{ \tilde D_k g(n_k) }{  \sqrt{ g(n_k) \ln \ln g(n_k) } }  & = \frac{ g(n_{k+1}) }{ g(n_k) } \frac{g(n_k)}{n_k/2} \frac{ g(n_k) }{  \sqrt{ g(n_k) \ln \ln g(n_k) } } \\
& \leq 2 (1 + \epsilon) \frac{ g(n_k)^{3/2} }{  n_k \sqrt{ \ln \ln g(n_k) } }\\
& = o(1).
\end{split}
\end{equation}
The last relationship follows, taking $g(n) = o( n^{2/3} )$.

We have
\begin{equation}
\begin{split}
\frac{ \tilde C_k g(n_k) }{  \sqrt{ g(n_k) \ln \ln g(n_k) } } & = 2\sigma_{i^*} (1+\epsilon) \frac{ g(n_{k+1}) }{ g(n_k) }  \sqrt{ \frac{ \ln \ln (n_k/2) }{ n_k }  \frac{ g(n_k) }{  \ln \ln g(n_k) } }\\
& \leq 2\sigma_{i^*} (1+\epsilon)^2 \sqrt{ \frac{ \ln \ln (n_k/2) }{ n_k }  \frac{ g(n_k) }{  \ln \ln g(n_k) } }\\
& = o(1).
\end{split}
\end{equation}
The last relationship follows, taking $g(n) = o( n / \ln \ln n )$.

We have
\begin{equation}
\begin{split}
\frac{ \tilde B_k g(n_k) }{  \sqrt{ g(n_k) \ln \ln g(n_k) } } & = \sigma_i \sqrt{2}(1+\epsilon)\left(2   + \frac{ g(n_{k+1}) }{ g(n_k) }  \right)\sqrt{ \frac{ \ln \ln \left(\frac{1-\epsilon}{\mu^* - \mu_i} g(n_k)\right) }{ \frac{1-\epsilon}{\mu^* - \mu_i} g(n_k ) } } \sqrt{ \frac{ g(n_k) }{  \ln \ln g(n_k) } }\\
& \leq \frac{ \sigma_i \sqrt{2}(1+\epsilon)\left(3   + \epsilon  \right)}{  \sqrt{ \frac{1-\epsilon}{\mu^* - \mu_i} } } \sqrt{ \frac{ \ln \ln \left(\frac{1-\epsilon}{\mu^* - \mu_i} g(n_k)\right) }{  \ln \ln g(n_k) } }\\
& = \frac{ \sigma_i \sqrt{2}(1+\epsilon)\left(3   + \epsilon  \right)}{  \sqrt{ \frac{1-\epsilon}{\mu^* - \mu_i} } } \left( 1 + o(1) \right).
\end{split}
\end{equation}
The last relationship follows, taking the $\{ n_k \}_{k \geq 0}$ as infinite and unbounded, and $g$ as increasing and unbounded.

We have
\begin{equation}
\begin{split}
\frac{ \tilde A_k g(n_k) }{  \sqrt{ g(n_k) \ln \ln g(n_k) } } & = (\mu^* - \mu_i)\left( \frac{g(n_{k+1})}{ g(n_k) } - 1\right)\sqrt{ \frac{ g(n_k) }{   \ln \ln g(n_k) } }.
\end{split}
\end{equation}
Let $\delta > 1$ by fixed. We use the bound here that for all positive $x \leq 1 - 1/\delta$, $1/(1-x) \leq 1 + \delta x$. Applying Eq. \eqref{eqn:g-ratio}, we have for sufficiently large $k$,
\begin{equation}
\begin{split}
\frac{g(n_{k+1})}{ g(n_k) } - 1 & \leq \frac{ 1 + \frac{ g'(n_k) }{ g(n_k) } }{ 1 - (1 + \epsilon) P_\Delta g'(n_k) } - 1 \\
&  \leq \left(1 + \frac{ g'(n_k) }{ g(n_k) }\right)\left(1 + \delta (1 + \epsilon) P_\Delta g'(n_k) \right) - 1\\
& = g'(n_k) \left( \delta (1 + \epsilon) P_\Delta  + o(1) \right).
\end{split}
\end{equation}
The last relationship follows, as $g' \rightarrow 0$ and $g \rightarrow \infty$ with $n_k$. Applying this to the above bound,
\begin{equation}
\begin{split}
\frac{ \tilde A_k g(n_k) }{  \sqrt{ g(n_k) \ln \ln g(n_k) } } & \leq (\mu^* - \mu_i) \left( \delta (1 + \epsilon) P_\Delta  + o(1) \right) g'(n_k) \sqrt{ \frac{ g(n_k) }{   \ln \ln g(n_k) } }\\
& = o(1).
\end{split}
\end{equation}
The last relationship follows, taking $g(n) = o( n^{2/3} )$.

Applying all of the above to the bound in Eq. \eqref{eqn:79}, this yields
\begin{equation}
\frac{ \Delta(n) }{ \sqrt{g(n) \ln \ln g(n)} } \leq  \left( \frac{ \sigma_i \sqrt{2}(1+\epsilon)\left(3   + \epsilon  \right)}{  \sqrt{ \frac{1-\epsilon}{\mu^* - \mu_i} } } \left( 1 - o(1) \right) + o(1)\right)  \frac{(1+\epsilon)^2}{ \left( \mu^* - \mu_i \right)},
\end{equation}
or
\begin{equation}
\limsup_n \frac{ \Delta(n) }{ \sqrt{g(n) \ln \ln g(n)} } \leq \left( \frac{ \sigma_i \sqrt{2}(1+\epsilon)\left(3   + \epsilon  \right)}{  \sqrt{ \frac{1-\epsilon}{\mu^* - \mu_i} } } \right)  \frac{(1+\epsilon)^2}{ \left( \mu^* - \mu_i \right)}.
\end{equation}
Taking the limit as $\epsilon \rightarrow 0$ completes the proof,
\begin{equation}
\limsup_n \frac{ g(n) - (\mu^* - \mu_i) T^i_\pi(n) }{ \sqrt{g(n) \ln \ln g(n)} } \leq \frac{ 3\sigma_i \sqrt{2}}{  \sqrt{ \mu^* - \mu_i } }.
\end{equation}


\end{proof}

\end{document}